\documentclass{article}
\usepackage[margin=0.85in]{geometry}
\usepackage{amsmath, amssymb, amsthm}
\usepackage{setspace}
\usepackage{algorithm}
\usepackage{algpseudocode}
\usepackage{enumitem}
\usepackage{graphicx}
\usepackage{tikz}
\usetikzlibrary{decorations.pathreplacing, arrows.meta, calc}

\usepackage{cite}

\usepackage{amsfonts}
\usepackage{times}
\usepackage{latexsym}
\usepackage{amssymb}
\usepackage{amsmath}
\usepackage{verbatim}
\def\bb0{{\mathbb{0}}}

\def\bb{{\mathbf{b}}}

\def\b0{{\mathbf{0}}}
\def\opt{\mathsf{OPT}}

\def\b1{{\mathbf{1}}}

\def\bbR{{\mathbb{R}}}

\def\cA{\mathcal{A}}

\def\cG{\mathcal{G}}

\def\cX{\mathcal{X}}

\def\sf0{{\mathsf{0}}}

\newtheorem{theorem}{Theorem}

\newtheorem{lemma}{Lemma}

\usepackage{tcolorbox}
\tcbuselibrary{theorems,skins,breakable}

\newtcolorbox{intuitionbox}[2][]{
  colback=blue!5!white,
  colframe=blue!60!black,
  fonttitle=\bfseries,
  title=Intuition: #2,
  #1,
  breakable
}

\newtcolorbox[
  auto counter,
  number within=section
]{Remarkbox}[2][]{
  colback=green!5!white,
  colframe=green!50!black,
  fonttitle=\bfseries,
  title=Remark~\thetcbcounter: #2,
  #1,
  breakable
}

\newtcolorbox[
  auto counter,
  number within=section
]{asmpbox}[2][]{
  colback=orange!8!white,
  colframe=orange!70!black,
  fonttitle=\bfseries,
  title=Assumption~\thetcbcounter: #2,
  #1,
  breakable
}

\newtcolorbox[
  auto counter,
  number within=section
]{propertybox}[2][]{
  colback=orange!8!white,
  colframe=orange!70!black,
  fonttitle=\bfseries,
  title=Property~\thetcbcounter: #2,
  #1,
  breakable
}

\newtcolorbox[
  auto counter,
  number within=section
]{defbox}[2][]{
  colback=gray!6!white,
  colframe=gray!50!black,
  fonttitle=\bfseries,
  title=Definition~\thetcbcounter: #2,
  #1,
  breakable
}

\newtcolorbox[
  auto counter,
  number within=section
]{examplebox}[2][]{
  colback=purple!5!white,
  colframe=purple!50!black,
  fonttitle=\bfseries,
  title=Example~\thetcbcounter: #2,
  #1,
  breakable
}

\begin{document}
\title{Convex Optimization with Nested Evolving Feasible Sets (CONES) under Time-Varying Loss Functions}
\author{Rahul Vaze
}
\maketitle              
\begin{abstract}
\emph{Convex Optimization with Nested Evolving Feasible Sets (CONES)} was introduced in \cite{CONESVaze} where the objective function \(f\) remains fixed but the feasible region evolves over time as a nested sequence \(S_1 \supseteq S_2 \supseteq \cdots \supseteq S_T\). The goal of an online algorithm is to simultaneously minimize the regret with respect to hindsight static optimal benchmark and the total movement cost $M_\cA(T)$ while ensuring feasibility at all times.
CONES is an optimization-oriented generalization of the well-known \emph{nested convex body chasing} (NCBC). In this paper, we extend CONES to allow for loss functions $f_t'$s to also change over time. 
When all loss functions are convex, we show that the projected proximal algorithm achieves $O(T^{1-\beta}), O(T^\beta)$ simultaneous regret and movement cost, respectively, for any $\beta \in [0,1)$,  over a time horizon of $T$. We also show that  any {\it weakly adaptive} online algorithm with $O(T^\beta)$ regret has a movement cost of $\Omega\left(T^{\frac{1-\beta}{2}}\right)$ for any $\beta \in [0,1)$.
When all loss functions are strongly convex,  we show that the projected proximal algorithm simultaneously achieves $O(1)$ regret and a movement cost of $O(\log T)$. To complement this, we show that any online algorithm with sublinear {\it anytime}  regret has a movement cost of $\Omega\left(\log T\right)$.
\end{abstract}
\section{Introduction}
In this paper, we consider the following problem known as CONES \cite{CONESVaze}.
At each round $t$, a convex (loss) function $f_t: \cX \rightarrow \bbR$ and a  convex set $S_t \subseteq \cX $ is revealed such that 
$S_t\subseteq S_{t-1}$, and  $S_0=\cX \subset \bbR^d$ is a convex, compact, and bounded set. Once $S_t$ is 
revealed, the objective for an online algorithm $\cA$ is to choose action $x_t \in S_t$ (feasible action) so as to {\bf simultaneously} minimize the 
loss function cost 
\begin{equation}\label{defn:optcostchangingf}
C_\cA(T) = \sum_{t=1}^T f_t(x_t),
\end{equation}
and total movement cost 
\begin{equation}\label{defn:mcostchangingf}
M_\cA(T) =   \sum_{t=1}^T  ||x_t - x_{t-1}||,
\end{equation}
where 
 $x_0 \in \cX$ is some fixed action. Throughout, we use $||.||$ to denote the $2$-norm or the Euclidean distance.

The performance of $\cA$ is compared against a {\it static} optimal benchmark $\opt$ that chooses its action $x^\opt \in S_T$ (assumed non-empty) that  minimizes the loss function cost, i.e., $$x^\opt \in  \arg  \min_{x\in S_T} \sum_t f_t(x)$$ and  
$$C_\opt(T) = \sum_{t=1}^T f_t(x^\opt).$$ Since $S_t$'s are nested, $x^\opt$ is feasible with respect to 
all $S_t's$.
The movement cost of $\opt$ is simply 
$$M_\opt(T) = ||x^\opt-x_0||.$$
The {\bf regret} of $\cA$ is then defined as 
\begin{equation}\label{defn:statregretchangingf}
\mathrm{Regret}_\cA(T)=\sup_{f_t, S_t} \{C_\cA(T) - C_\opt(T)\}.
\end{equation} 
For $\cA$ its two objectives are: minimize $\mathrm{Regret}_\cA(T)$ and movement cost $M_\cA(T)$ simultaneously. 


\subsection{Prior Work with $f_t=f\ \forall \ t$}
CONES was introduced in \cite{CONESVaze} for the case when $f_t=f\ \forall \ t$. \cite{CONESVaze} also discussed in detail the motivation to study it, its applications, and its connections to related prior work, e.g. in nested convex body chasing (NCBC) 
\cite{linial, bansal2017nestedconvexbodieschaseable, argue2018nearlylinearboundchasingnested, bubeck2021chasingnestedconvexbodies, guptatang, Sellke2023}, convex online constrained optimization (COCO) \cite{yu2017online, pmlr-v70-sun17a, yi2023distributed, neely2017online, georgios-cautious, guo2022online, Sinha2024,Vaze2025b}, sensitivity/perturbation analysis of optimization problems \cite{bonnans2000perturbation}, trajectory tracking \cite{simonetto2020time}, etc.
With $f_t=f\ \forall \ t$ the following results were derived in \cite{CONESVaze}.
\begin{enumerate}[leftmargin=*,nosep]
\item When $f$ is {\bf strongly convex}, it showed that the greedy algorithm $\cG$ that chooses $x_t=x_t^\star = \arg \min_{x\in S_t} f(x)$ for all $t$ has minimum regret among all online algorithms but incurs a movement cost of $\Theta(\sqrt{T})$. Next, a  `lazy' algorithm called \textsc{Frugal} was analyzed, that plays action  $x_t= x_t^\star$ only when it is forced to do so in order to keep its regret non-positive and otherwise plays the action obtained by projecting its current action $x_{t-1}$
on to the most recently revealed set $S_t$. The regret of \textsc{Frugal} is non-positive by definition and its movement cost is shown to be $O(\log T)$, a fundamental improvement over $\cG$ when $f$ is also smooth. This upper bound was complemented by a lower bound, that showed that any online algorithm with sublinear {\it anytime}  regret has movement cost $\Omega\left(\sqrt{\frac{\log{T}}{\log \log T}}\right)$. Recently, an algorithm that combines \textsc{Frugal}  and Level Set Projection (LSP) algorithm of \cite{CONESVaze} has been shown to achieve a non-positive regret and movement cost of $O(\sqrt{\log T})$ in \cite{CONESDhruv}.
\item When $f$ is {\bf convex}, \cite{CONESVaze} showed a  negative result that the greedy algorithm $\cG$ has a movement cost of $\Omega(T)$. Next, it showed that the LSP algorithm  achieves a (regret, movement cost) tuple of $(T^{1-\beta}, T^{\beta})$ for any $\beta \in [0,1]$. 
\end{enumerate}
\vspace{0.1in}
In this paper, with arbitrary time-varying $f_t$'s, we derive the following results.
\subsection{Contributions}
\begin{enumerate}[leftmargin=*,nosep]
\item When $f_t$'s are {\bf strongly convex}, we show that the projected proximal algorithm  (we refer to it as \textrm{Prox}) that chooses  
\begin{equation}\label{defn:algprox}
x_t
\in
\arg\min_{x\in S_t}
\left\{
f_t(x)+\frac{1}{2\eta_t}\|x-x_{t-1}\|^2
\right\},
\end{equation} has  regret of  $O(1)$ and its movement cost is $O(\log T)$, when 
$\eta_t=\frac{1}{\mu t}$.   \textrm{Prox} is conceptually a different algorithm than \textsc{Frugal} \cite{CONESVaze} and automatically adapts to changing constraints rather than being forced depending on current regret.

We also show that any online algorithm with sublinear {\it anytime}  regret has movement cost $\Omega\left(\log T\right)$. Thus obtaining a tight result and showing that \textrm{Prox} is an optimal (order-wise) algorithm.

It is worth noting that compared to the case when $f_t=f$ for all $t$, the optimal movement cost with sublinear {\it anytime} regret is order-wise different.

\item When $f$ is {\bf convex}, we show that \textrm{Prox} \eqref{defn:algprox} achieves a  (regret, movement cost) tuple of $(T^{1-\beta}, T^{\beta})$ for any $\beta \in [0,1]$ by choosing $\eta=\frac{D^2}{\epsilon T}$ and $\epsilon=T^{-\beta}$. 

Moreover, we also show that any {\it weakly adaptive} online algorithm with $O(T^\beta)$ regret has a movement cost of $\Omega\left(T^{\frac{1-\beta}{2}}\right)$ for $\beta \in [0,1)$ {\bf even when $f_t=f$ for all $t$}, the special case studied in \cite{CONESVaze}. No such  lower bound was provided in  \cite{CONESVaze}.
 \end{enumerate}
 \vspace{0.1in}

\section{$f_t$'s are strongly convex}
In this section, we consider the twin objectives of simultaneously minimizing regret \eqref{defn:statregretchangingf} and movement cost \eqref{defn:mcostchangingf}, when $f_t$'s are $\mu$-strongly convex and $G$-Lipschitz. Throughout, 
we assume that the diameter of $\cX$ is at most $D$.
\subsection{Upper Bound}

\noindent
\fcolorbox{black}{gray!6}{%
\begin{minipage}{0.97\textwidth}
\begin{theorem}\label{thm:ubstrcvx}
When $f_t$'s are $\mu$-strongly convex, $G$-Lipschitz, for algorithm \textrm{Prox} \eqref{defn:algprox} with 
$\eta_t=\frac{1}{\mu t}$,
\begin{equation}
\mathrm{Regret}_\cA(T)
\le
\frac{\mu D^2}{2}
=O(1),
\end{equation}

and
\[
\boxed{
M_\cA(T)
\le
C_d\left(
D+\frac{G}{\mu}(1+\log T)
\right)
=
O(\log T),
}
\]
where $C_d$ is a constant depending only on $||.||$ and the dimension $d$. The exact bound on $C_d$ can be extracted from \cite{stepanov2017self}.
\end{theorem}
\end{minipage}}
\vspace{0.5em}

The main idea of the proof is to exploit the strong convexity of $f_t$'s to upper bound the regret by an $O(1)$ quantity, while using the fact that (proved in Lemma \ref{lem:algequiv})
$x_t$ defined in \eqref{defn:algprox} is equivalent to 
\begin{equation}\label{eq:equiv}
x_t = \Pi_{S_t}\bigl(x_{t-1}-\eta_t g_t\bigr).
\end{equation}
where $g_t\in\partial f_t(x_t)$, and $\partial f_t(x_t)$ denotes the
subgradient set of $f_t$ at $x_t$.  This association \eqref{eq:equiv} allows us to utilize a recent result in \cite{SinhaGeom} that upper bounds $\sum_t ||x_t-x_{t-1}||$ for iterates $x_t$ satisfying \eqref{eq:equiv} when sets $S_t$'s are nested.  Remarkably, it is worth noting that $\eta_t$ is in the numerator of \eqref{eq:equiv} unlike being in the denominator of \eqref{defn:algprox}. Thus, with $\eta_t=\frac{1}{\mu t}$ allows us to bound the movement cost as $O(\log T)$.

\begin{Remarkbox}{}
The regret bound of Theorem \ref{thm:ubstrcvx}
\[
\mathrm{Regret}_{\cA}(T)\le \frac{\mu D^2}{2}
\]
has a qualitatively different dependence on the strong-convexity parameter $\mu$ from the usual $O(1/\mu)$ regret bound for online convex optimization \cite{HazanBook} or COCO \cite{SinhaGeom}. This is because algorithm \textrm{Prox} uses the entire loss function $f_t$ in its implicit proximal update and chooses
\[
\eta_t=\frac{1}{\mu t}.
\]
As $\mu$ decreases, the proximal regularization becomes weaker and the update approaches the greedy minimizer of the current loss over $S_t$. In the present CONES setting, where $f_t$ and $S_t$ are revealed before $x_t$ is chosen, the greedy action satisfies
\[
f_t(x_t)\le f_t(x^\opt)
\]
at every round, since $x^\opt\in S_t$. Thus the regret can in fact approach zero as $\mu\to0$. The dependence
\[
\mathrm{Regret}_{\cA}(T)=O(\mu D^2)
\]
is therefore consistent with the algorithm becoming increasingly greedy as the strong convexity vanishes. 

\end{Remarkbox}

To complement this upper bound result, we derive a matching lower bound as follows.
\subsection{Lower Bound}

\noindent
\fcolorbox{black}{gray!6}{%
\begin{minipage}{0.97\textwidth}
\begin{theorem}\label{thm:lbstrcvx}
Suppose that an online algorithm $\mathcal A$ satisfies, for every $t$,
\begin{equation}\label{eq:anytimeregret}
\mathrm{Regret}_\cA(t)
:=
 \sup_{\{f_r,S_r\}_{r=1}^t} \sum_{r=1}^t f_r(x_r)
-
\min_{x\in\mathcal S_t}\sum_{r=1}^t f_r(x)
\le R(t).
\end{equation}

An online algorithm $\cA$ is defined to have {\it anytime} sublinear regret if for all $t$,
\[
R(t)=o(t),
\]
For any $\cA$ with {\it anytime} sublinear regret, against an adaptive adversary, there exists a sequence of
$1$-strongly convex and $1$-Lipschitz loss functions $f_t$ for which
\[
M_{\mathcal A}(T)
=
\sum_{t=1}^T \|x_t-x_{t-1}\|
=
\Omega(\log T).
\]
\end{theorem}
\end{minipage}}
\vspace{0.5em}

{\it Discussion:} Theorem \ref{thm:ubstrcvx} and \ref{thm:lbstrcvx} together fully characterize the optimal movement cost for any algorithm that can achieve {\it anytime} $o(T)$ regret. Compared to the case when $f_t=f$ for all $t$, where the movement cost is $O(\sqrt{\log T})$ and $\Omega\left(\sqrt{\frac{\log T}{\log \log T}}\right)$ \cite{CONESVaze, CONESDhruv} for any algorithm that can achieve {\it anytime} $o(T)$ regret, with changing $f_t$'s it is  $\Theta(\log T)$. Thus, there is a fundamental 
degradation which is expected in this more general setting.

\section{$f_t$'s are convex}
In this section, we consider the twin objectives of minimizing regret \eqref{defn:statregretchangingf} and movement cost \eqref{defn:mcostchangingf}, when $f_t$'s are just convex and $G$-Lipschitz.
\subsection{Upper Bound}

\noindent
\fcolorbox{black}{gray!6}{%
\begin{minipage}{0.97\textwidth}
\begin{theorem}\label{thm:ubcvx}
When $f_t$'s are convex and $G$-Lipschitz, 
 for any $\epsilon>0$, algorithm \textrm{Prox} \eqref{defn:algprox}, with 
\[
\eta=\frac{D^2}{\epsilon T}.
\]
satisfies
\[
\mathrm{Regret}_\cA(T)
\le
\frac{D^2}{2\eta}
=
\frac{\epsilon T}{2},
\]
and
\[
M_\cA(T)
\le
C_d\left(
D+\frac{GD^2}{\epsilon}
\right),
\]
where $C_d$ is the same constant that appears in Theorem \ref{thm:ubstrcvx}.
In particular, for any $0\le\beta\le1$, choosing
\[
\epsilon=T^{-\beta}
\]
 yields
\[
\mathrm{Regret}_\cA(T)=O(T^{1-\beta}),
\qquad
M_\cA(T)=O_d(T^\beta).
\]
\end{theorem}
\end{minipage}}
\vspace{0.5em}

We next derive lower bound on achievable simultaneous regret and movement cost when $f_t$'s are convex.

\subsection{Lower Bounds}
\subsubsection{Weakly Adaptive Algorithms}

\begin{defbox}{Weakly adaptive algorithm}
An online algorithm $\mathcal A$ for CONES with changing losses
$f_1,\ldots,f_T$ is called \emph{$R$-weakly adaptive} if, for every
CONES instance, every interval
\[
I=\{r,r+1,\ldots,s\}\subseteq[T],
\]
with benchmark
\[
x_s^\opt
\in
\arg\min_{x\in S_s}\sum_{t=1}^s f_t(x)
\]
$\cA$ satisfies
\[
\sum_{t=r}^{s}
\bigl(f_t(x_t)-f_t(x_s^\opt)\bigr)
\le R.
\]
Since the feasible sets are nested,
\[
S_s\subseteq S_t
\qquad\text{for all }t\le s,
\]
and hence
\[
x_s^\opt\in S_t
\qquad\text{for all }t\in[a,b].
\]
Thus, $x_s^\opt$ is feasible throughout the interval $I$.
Taking $r=1$ and $s=T$ gives
\[
\sum_{t=1}^T
\bigl(f_t(x_t)-f_t(x_T^\opt)\bigr)
\le R,
\]
which is the usual terminal static-regret guarantee.

\end{defbox}

The notion of weakly adaptive algorithms is quite popular in related literature \cite{bernasconi2024no, SinhaGeom}.
We next derive a lower bound on the simultaneous regret and movement cost for any weakly adaptive algorithm {\bf when $f_t=f$ for all $t$}, which directly applies for general $f_t$'s, and thereafter give examples of weakly adaptive algorithms.


\noindent
\fcolorbox{black}{gray!6}{%
\begin{minipage}{0.97\textwidth}
\begin{theorem}
\label{lemma:lbgenConvex}
Fix $T\ge1$ and let
\[
R:=\mathrm{Regret}_{\cA}(T)\ge0, \quad \text{satisfy} \quad
T\ge16(R+1).
\]
Consider any online algorithm $\cA$ satisfying the weakly adaptive regret guarantee
\[
\sum_{t=r}^{s}
\bigl(f_t(x_t)-f_t(x_s^\opt)\bigr)
\le R
\qquad
\text{for every }[r,s]\subseteq[T], \quad \text{where} \quad 
x_s^\opt\in\arg\min_{x\in S_s}\sum_{t=1}^s f_t(x).
\]
Then there exists a two-dimensional CONES instance with
\[
\mathcal X=[-1,1]\times[0,1],
\qquad
\operatorname{diam}(\mathcal X)=\sqrt{5},
\qquad
S_0=\mathcal X,
\qquad
S_t\subseteq S_{t-1},
\]
and fixed loss $f:\mathcal X\to\mathbb R$ that is convex and $1$-Lipschitz, such that
\[
\boxed{
M_{\cA}(T)
\ge
\frac18\sqrt{\frac{T}{R+1}}-1.
}
\]
In particular, if
\[
\mathrm{Regret}_{\cA}(T)=O(T^\beta),
\qquad
0\le \beta <1,
\]
then
\[
\boxed{
M_{\cA}(T)
=
\Omega\!\left(T^{(1-\beta)/2}\right).
}
\]
\end{theorem}
%
\end{minipage}}
\vspace{0.5em}

\subsection{Examples of weakly adaptive algorithms for CONES}
\begin{enumerate}

\item \textrm{Prox} algorithm \eqref{defn:algprox} is $R$-weakly adaptive, with
\[
R=\frac{D^2}{2\eta},
\]
when $f_t:\mathcal X\to\mathbb R$ are convex and Prox uses a
constant stepsize $\eta_t=\eta$ for all $t$.

We show it as follows.
Fix any interval $I=[a,b]\subseteq[T]$, and let
\[
x_b^\opt\in\arg\min_{x\in S_b}\sum_{t=1}^b f_t(x).
\]
Since the feasible sets are nested,
\[
S_b\subseteq S_t
\qquad\text{for all }t\le b,
\]
and hence
\[
x_b^\opt\in S_t
\qquad\text{for all }t\in[a,b].
\]

Let $g_t\in\partial f_t(x_t)$. Since $f_t$ is convex,
\[
f_t(x_t)-f_t(x_b^\opt)
\le
\langle g_t,x_t-x_b^\opt\rangle.
\]
The first-order optimality condition for the proximal update
\eqref{defn:algprox} gives
\[
\left\langle
g_t+\frac{x_t-x_{t-1}}{\eta},
x_b^\opt-x_t
\right\rangle
\ge 0.
\]
Therefore,
\[
\langle g_t,x_t-x_b^\opt\rangle
\le
\frac{1}{\eta}
\langle x_t-x_{t-1},x_b^\opt-x_t\rangle.
\]
Using
\[
2\langle x_t-x_{t-1},x_t-x_b^\opt\rangle
=
\|x_t-x_{t-1}\|^2
+\|x_t-x_b^\opt\|^2
-\|x_{t-1}-x_b^\opt\|^2,
\]
we obtain
\[
f_t(x_t)-f_t(x_b^\opt)
\le
\frac{
\|x_{t-1}-x_b^\opt\|^2
-\|x_t-x_b^\opt\|^2
-\|x_t-x_{t-1}\|^2
}{2\eta}.
\]
Summing over $t=a,\ldots,b$ and dropping the nonpositive movement term gives
\[
\begin{aligned}
\sum_{t=a}^b
\bigl(f_t(x_t)-f_t(x_b^\opt)\bigr)
&\le
\frac{
\|x_{a-1}-x_b^\opt\|^2
-\|x_b-x_b^\opt\|^2
}{2\eta} \le
\frac{D^2}{2\eta}.
\end{aligned}
\]
Since this holds for every interval $[a,b]\subseteq[T]$, the Prox algorithm is
$R$-weakly adaptive with
\[
R=\frac{D^2}{2\eta}.
\]

\item
With changing losses $f_t$, the \textsc{Greedy} algorithm plays
\[
x_t\in\arg\min_{x\in S_t} f_t(x)
\]
at every round.

Consider any interval $I=[a,b]\subseteq[T]$. By definition,
\[
x_b^\opt
\in
\arg\min_{x\in S_b}\sum_{t=1}^b f_t(x).
\]
Since the feasible sets are nested,
\[
S_b\subseteq S_t
\qquad\forall t\le b,
\]
and hence
\[
x_b^\opt\in S_t
\qquad\forall t\in[a,b].
\]
By the definition of \textsc{Greedy},
\[
f_t(x_t)
=
\min_{x\in S_t}f_t(x)
\le
f_t(x_b^\opt)
\qquad\forall t\in[a,b].
\]
Therefore,
\[
\sum_{t=a}^b
\bigl(f_t(x_t)-f_t(x_b^\opt)\bigr)
\le0.
\]
Hence \textsc{Greedy} is $0$-weakly adaptive.

\item Level Set Projection (LSP) algorithm \cite{CONESVaze}:
The LSP algorithm was defined in \cite{CONESVaze} for the case
$f_t=f$ for all $t$. Its natural extension to changing losses is as
follows. Define
\[
L_t:=\min_{x\in S_t}f_t(x),
\qquad
\Phi_t:=
\left\{
x\in S_t:
f_t(x)\le L_t+\varepsilon
\right\},
\]
and choose some
\[
x_t\in\Phi_t.
\]
For example, one may choose
\[
x_t=\Pi_{\Phi_t}(x_{t-1}).
\]
By construction,
\[
f_t(x_t)
\le
\min_{x\in S_t}f_t(x)+\varepsilon.
\]

\begin{lemma}
The changing-loss version of \textsc{LSP} is
$T\varepsilon$-weakly adaptive.
\end{lemma}

\begin{proof}
Consider any interval
\[
I=[a,b]\subseteq[T].
\]
By definition,
\[
x_b^\opt
\in
\arg\min_{x\in S_b}\sum_{t=1}^b f_t(x).
\]
Since the feasible sets are nested, 
$S_b\subseteq S_t \forall t\le b,$
and hence $
x_b^\opt\in S_t \forall t\in[a,b].$

By construction of $x_t$,
\[
f_t(x_t)
\le
\min_{x\in S_t}f_t(x)+\varepsilon
\le
f_t(x_b^\opt)+\varepsilon.
\]
Thus,
\[
f_t(x_t)-f_t(x_b^\opt)
\le\varepsilon
\qquad\forall t\in[a,b].
\]
Summing over the interval gives
\[
\sum_{t=a}^b
\bigl(f_t(x_t)-f_t(x_b^\opt)\bigr)
\le
|I|\varepsilon
\le
T\varepsilon.
\]
Hence the changing-loss version of \textsc{LSP} is
$T\varepsilon$-weakly adaptive.
\end{proof}

\end{enumerate}


\section{Conclusion}

In this paper, we generalized the important optimization paradigm of Convex
Optimization with Nested Evolving Feasible Sets (CONES), introduced in
\cite{CONESVaze}, from a fixed objective function to a setting in which both
the loss functions and the feasible sets may evolve over time, while the
benchmark remains a static hindsight optimum. This extension allows us to
separate the effects of evolving optimization geometry from those arising
from temporal variation in the objective itself.

For the strongly convex case, changing loss functions increases the
movement required under $o(T)$ anytime regret from the previously known
$O(\sqrt{\log T})$ upper bound in the fixed-objective setting to the
tight $\Theta(\log T)$ rate established here. For the convex case, we 
establish the same upper bound in the evolving objective setting as already known for the fixed-objective setting, while developing 
new lower bounds that apply even for the fixed-objective setting. With these improvements we have widened the scope of CONES and 
developed a better theoretical understanding.
\bibliographystyle{plain}
\newpage
\bibliography{oco,references}
\section{Proof of Theorem \ref{thm:ubstrcvx}}

\begin{proof}
Let
\[
g_t\in\partial f_t(x_t),
\]
where $\partial f_t(x_t)$ denotes the sub-gradient of $f_t$ at $x_t$.
Since $f_t$ is $\mu$-strongly convex,
\begin{equation}\label{eq:strcvx1}
f_t(x_t)-f_t(x^\opt)
\le
\langle g_t,x_t-x^\opt\rangle
-\frac{\mu}{2}\|x_t-x^\opt\|^2.
\end{equation}

The first-order optimality condition for algorithm \textrm{Prox} \eqref{defn:algprox} implies
\[
\left\langle
g_t+\frac{x_t-x_{t-1}}{\eta_t},
x^\opt-x_t
\right\rangle\ge0,
\]
since $x^\opt \in S_t$ for all $t$.

Therefore,
\[
\langle g_t,x_t-x^\opt\rangle
\le
\frac{1}{\eta_t}
\langle x_t-x_{t-1},x^\opt-x_t\rangle.
\]

Using
\[
\langle x_t-x_{t-1},x^\opt-x_t\rangle
=
-\langle x_t-x_{t-1},x_t-x^\opt\rangle
\]
and the  identity
\[
2\langle x_t-x_{t-1},x_t-x^\opt\rangle
=
\|x_t-x^\opt\|^2
+
\|x_t-x_{t-1}\|^2
-
\|x_{t-1}-x^\opt\|^2,
\]
we obtain
\begin{equation}\label{eq:proxmanipulation}
\langle g_t,x_t-x^\opt\rangle
\le
\frac{1}{2\eta_t}
\left(
\|x_{t-1}-x^\opt\|^2
-
\|x_t-x^\opt\|^2
-
\|x_t-x_{t-1}\|^2
\right).
\end{equation}

Substituting this into \eqref{eq:strcvx1} and using $\eta_t=1/(\mu t)$,
\[
\begin{aligned}
f_t(x_t)-f_t(x^\opt)
\le \frac{\mu}{2}\Big(
&t\|x_{t-1}-x^\opt\| ^2
-(t+1)\|x_t-x^\opt\| ^2 -t\|x_t-x_{t-1}\| ^2
\Big).
\end{aligned}
\]
Summing over $t$ yields
\[
\mathrm{Regret}_\cA(T)
\le
\frac{\mu}{2}
\left(
\|x_0-x^\opt\| ^2
-(T+1)\|x_T-x^\opt\| ^2
-\sum_{t=1}^T t\|x_t-x_{t-1}\| ^2
\right),
\]
and therefore
\[
\mathrm{Regret}_\cA(T)
\le
\frac{\mu D^2}{2}.
\]

For bounding the movement cost $M_\cA(T)$, we rewrite \textrm{Prox} \eqref{defn:algprox} as follows.

\begin{lemma}\label{lem:algequiv}
$x_t$ defined in \eqref{defn:algprox} is equivalent to 
\begin{equation}\label{defn:projupdate}
x_t = \Pi_{S_t}\bigl(x_{t-1}-\eta_t g_t\bigr).
\end{equation}
 where $g_t\in\partial f_t(x_t)$. 
 \end{lemma}
Note that \eqref{defn:projupdate} is a fixed point equation where $x_t$ appears on both sides of the equality since $g_t\in\partial f_t(x_t)$. This fixed point property is only a structural result and not used by the \textrm{Prox} algorithm. 
Moreover, importantly $\eta_t$ appears in the numerator \eqref{defn:projupdate} compared to the denominator \eqref{defn:algprox}, which allows a {\it win-win} on two fronts given that $\eta_t=\frac{1}{\mu t}$: $T$ independent regret bound and $O(\log T)$ movement cost.

Proof of Lemma \ref{lem:algequiv} is elementary and provided later for completeness.

Thus, the update  \eqref{defn:algprox} of algorithm \textrm{Prox} is of the form
\begin{equation}\label{defn:updaterewrite}
x_t
=
\Pi_{S_t}\bigl(x_{t-1}-e_t\bigr),
\qquad \text{where} \ 
e_t:=\eta_t g_t,
\end{equation}
and 
\[
\|e_t\| \le\frac{G}{\mu t},
\]
since $\eta_t=1/(\mu t)$ and $||g_t|| \le G$.

For a nested projected-gradient trajectory \eqref{defn:updaterewrite} we invoke the following result. 
\begin{lemma}[Nested-projection movement bound
{\cite[ Lemma~4.1]{SinhaGeom}}]
\label{lem:totalpertprojlength}
Let
\[
K_0\supseteq K_1\supseteq\cdots\supseteq K_T
\]
be a nested sequence of nonempty closed convex subsets of $\mathbb R^d$
such that
\[
\operatorname{diam}(K_0)\le D.
\]
Let $z_0\in K_0$, and suppose that
\[
z_t
=
\Pi_{K_t}(z_{t-1}-e_t),
\qquad
t=1,\ldots,T,
\]
where $e_t\in\mathbb R^d$ is arbitrary. Then
\[
\boxed{
\sum_{t=1}^T\|z_t-z_{t-1}\|
\le
C_d
\left(
D+\sum_{t=1}^T\|e_t\|
\right),
}
\]
where $C_d>0$  is a constant depending only on $||.||$ and the dimension $d$.
\end{lemma}

In our setting $S_t = K_t$ and $S_t \subseteq S_{t-1}$, thus, using Lemma \ref{lem:totalpertprojlength}, we get from \eqref{defn:updaterewrite}, that for algorithm \textrm{Prox},
\[
\sum_{t=1}^T\|x_t-x_{t-1}\| 
\le
C_d\left(
D+\sum_{t=1}^T\|e_t\| 
\right),
\]
with $\|e_t\| \le\frac{G}{\mu t}$.
Hence
\[
M_\cA(T)
\le
C_d\left(
D+\frac{G}{\mu}\sum_{t=1}^T\frac1t
\right)
\le
C_d\left(
D+\frac{G}{\mu}(1+\log T)
\right).
\]
\end{proof}

 \begin{proof}[Proof of Lemma \ref{lem:algequiv}]
Let $y_t:=x_{t-1}-\eta_t g_t$.
From \eqref{defn:algprox}, $x_t$ minimizes the function
\[
\phi_t(x):=
f_t(x)+\frac{1}{2\eta_t}\|x-x_{t-1}\| ^2
\]
over the closed convex set $S_t$. Thus, the first-order optimality condition implies
\[
0\in \partial \phi_t(x_t)+N_{S_t}(x_t),
\]
where $N_S(x)$ is the normal cone to $S$ at point $x\in S$.
Because
\[
\partial \phi_t(x_t)
=
g_t+\frac{1}{\eta_t}(x_t-x_{t-1}),
\]
we obtain
\[
0\in
g_t+\frac{1}{\eta_t}(x_t-x_{t-1})
+N_{S_t}(x_t).
\]

Thus there exists some
\[
n_t\in N_{S_t}(x_t)
\]
such that
\[
g_t+\frac{1}{\eta_t}(x_t-x_{t-1})+n_t=0.
\]
Multiplying by $\eta_t$ gives
\[
\eta_t g_t+x_t-x_{t-1}+\eta_t n_t=0,
\]
and hence
\[
x_{t-1}-\eta_t g_t-x_t
=
\eta_t n_t.
\]
By the definition of $y_t$,
\[
y_t-x_t=\eta_t n_t.
\]
Since the normal cone is a cone, $n_t\in N_{S_t}(x_t)$ implies
\[
\eta_t n_t\in N_{S_t}(x_t).
\]
Therefore
\begin{equation}
\label{eq:dummy1}
y_t-x_t\in N_{S_t}(x_t).
\end{equation}

We next recall the well-known projection characterization, whose proof is included for completeness. Let $S$ be a nonempty closed
convex set. We claim that, for $x\in S$,
\begin{equation}\label{eq:dummy2}
x=\Pi_S(y)
\iff
y-x\in N_S(x).
\end{equation}

First suppose that $x=\Pi_S(y)$. By definition,
\[
\|y-x\| ^2
\le
\|y-z\| ^2
\qquad
\forall z\in S.
\]
For any $z\in S$ and any $\lambda\in(0,1]$, convexity gives
\[
x+\lambda(z-x)\in S.
\]
Hence
\[
\|y-x\| ^2
\le
\|y-x-\lambda(z-x)\| ^2.
\]
Expanding,
\[
0
\le
-2\lambda\langle y-x,z-x\rangle
+\lambda^2\|z-x\| ^2.
\]
Dividing by $\lambda$ and letting $\lambda\downarrow0$ gives
\[
\langle y-x,z-x\rangle\le0
\qquad
\forall z\in S.
\]
By the definition of the normal cone,
\[
y-x\in N_S(x).
\]

Conversely, suppose that $x\in S$ and
\[
y-x\in N_S(x).
\]
Then, for every $z\in S$,
\[
\langle y-x,z-x\rangle\le0.
\]
Therefore
\[
\begin{aligned}
\|y-z\| ^2
&=
\|y-x-(z-x)\| ^2\\
&=
\|y-x\| ^2
+\|z-x\| ^2
-2\langle y-x,z-x\rangle\\
&\ge
\|y-x\| ^2.
\end{aligned}
\]
Thus $x$ minimizes $\|y-z\| $ over $z\in S$, so
\[
x=\Pi_S(y).
\]
This proves \eqref{eq:dummy2}.

Finally, $x_t\in S_t$ by construction, and \eqref{eq:dummy1} gives
\[
y_t-x_t\in N_{S_t}(x_t).
\]
Applying \eqref{eq:dummy2} with $S=S_t$, $y=y_t$, and $x=x_t$ yields
\[
x_t=\Pi_{S_t}(y_t)
=
\Pi_{S_t}\bigl(x_{t-1}-\eta_t g_t\bigr).
\]
\end{proof}

\section{Proof of Theorem \ref{thm:lbstrcvx}}
\begin{proof}
Let the mother set be
\[
\mathcal X=[-1/2,1/2],
\]
and let
\[
S_t=\mathcal X
\qquad\forall t, i.e.,
\]
{\bf all sets $S_t$'s are identical and hence nested}.

Let
\[
N_k=33^k,\qquad k\ge0,
\]
so that $N_0=1$. Define the first phase by
\[
\mathcal I_1=\{1,\ldots,N_1\},
\]
and, for every $k\ge2$, define
\[
\mathcal I_k
=
\{N_{k-1}+1,\ldots,N_k\}.
\]
Thus, for every $k\ge2$, if
\[
N=N_{k-1},
\]
then phase $k$ consists of exactly
\[
N_k-N_{k-1}=32N
\]
rounds.

Fix
\[
s_1=\frac12
\]
in advance, and for every $t\in\mathcal I_1$ define
\[
f_t(x)=\frac12(x-s_1)^2.
\]

At the beginning of phase $k\ge2$, the phase starts at round
$N_{k-1}+1$. Recall that, in CONES, the action $x_t$ is chosen after $f_t$ and
$S_t$ are revealed. Hence, at the beginning of phase $k$, the adversary
can observe $x_{N_{k-1}}$ before choosing $s_k$. 

After observing action $x_{N_{k-1}}$ of $\cA$, the adversary chooses
\[
s_k\in\{-1/2,1/2\}
\]
such that
\[
||x_{N_{k-1}}-s_k||\ge\frac12,
\]
and then for every $t\in\mathcal I_k$, uses
\[
f_t(x)=\frac12(x-s_k)^2,
\qquad
t=N_{k-1}+1,\ldots,N_k.
\]

Since $x,s_k\in[-1/2,1/2]$,
\[
||f_t'(x)||=||x-s_k||\le1,
\]
so every $f_t$ is $1$-Lipschitz. Moreover, every $f_t$ is
$1$-strongly convex.

Next, we show that for any $\cA$ with $o(t)$ {\it anytime} regret, for all sufficiently large $N$, any phase of length
$32N$ must incur at least $1/4$ movement. Fix a phase $k\ge2$ and write
\[
N=N_{k-1}.
\]
Suppose, for contradiction, that the total movement in phase $k$
\[
\sum_{t=N+1}^{N+32N}||x_t-x_{t-1}||
<
\frac14.
\]
Then, for every
\[
t\in\{N+1,\ldots,N+32N\},
\]
\[
||x_t-x_N||
\le
\sum_{r=N+1}^{t}||x_r-x_{r-1}||
<
\frac14.
\]
By the choice of $s_k$,
\[
||x_t-s_k||
\ge
||x_N-s_k||-||x_t-x_N||
>
\frac12-\frac14
=
\frac14.
\]
Hence
\begin{equation}\label{eq:phase-action-loss}
f_t(x_t)
=
\frac12(x_t-s_k)^2
>
\frac1{32}
\end{equation}
for every
\[
t=N+1,\ldots,N+32N.
\]

Consider the prefix $[1,\tau]$ where
\[
\tau=N+32N=33N.
\]
Let
\[
x_\tau^\opt
\in
\arg\min_{x\in\mathcal X}
\sum_{r=1}^{\tau}f_r(x)
\]
be the static benchmark for this prefix, since $S_t=\cX$ for all $t$.

Every loss has the form
\[
f_r(x)=\frac12(x-s_r)^2,
\qquad
s_r\in\{-1/2,1/2\}.
\]
Hence
\[
\sum_{r=1}^{\tau}f_r(x)
=
\frac12\sum_{r=1}^{\tau}(x-s_r)^2,
\]
whose minimizer is
\[
x_\tau^\opt
=
\frac1{\tau}\sum_{r=1}^{\tau}s_r.
\]
Over the first $N$ rounds, the offsets $s_r$ are those prescribed by phases
$1,\ldots,k-1$, while over the remaining $32N$ rounds, corresponding to
the current phase $k$, the offset is $s_k$. Therefore,
\[
\begin{aligned}
x_\tau^\opt-s_k
&=
\frac1{33N}
\left(
\sum_{r=1}^{N}s_r
+
32N s_k
\right)
-s_k\\
&=
\frac1{33N}
\left(
\sum_{r=1}^{N}s_r
+
32N s_k
-
33N s_k
\right)\\
&=
\frac1{33N}
\sum_{r=1}^{N}(s_r-s_k).
\end{aligned}
\]
Since
\[
||s_r-s_k||\le1,
\]
we obtain
\[
||x_\tau^\opt-s_k||
\le
\frac{N}{33N}
=
\frac1{33}
<
\frac18.
\]
Consequently, for every
\[
t=N+1,\ldots,N+32N,
\]
\[
f_t(x_\tau^\opt)
=
\frac12(x_\tau^\opt-s_k)^2
<
\frac1{128}.
\]
Combining this with \eqref{eq:phase-action-loss},
\begin{equation}\label{eq:phase-regret-per-round}
f_t(x_t)-f_t(x_\tau^\opt)
>
\frac1{32}-\frac1{128}
=
\frac3{128}.
\end{equation}

For the preceding $N$ rounds,
\[
r=1,\ldots,N,
\]
each loss satisfies
\[
f_r(x_r)\ge0
\]
and, since $x_\tau^\opt,s_r\in[-1/2,1/2]$,
\[
f_r(x_\tau^\opt)
=
\frac12(x_\tau^\opt-s_r)^2
\le\frac12.
\]
Thus
\begin{equation}\label{eq:phase-regret-first}
f_r(x_r)-f_r(x_\tau^\opt)
\ge-\frac12,
\qquad
r=1,\ldots,N.
\end{equation}

Splitting the prefix regret into these two parts and using
\eqref{eq:phase-regret-first} and \eqref{eq:phase-regret-per-round},
\begin{align}
R_{\mathcal A}(\tau)
&=
\sum_{r=1}^{N}
\bigl(f_r(x_r)-f_r(x_\tau^\opt)\bigr)
+
\sum_{r=N+1}^{N+32N}
\bigl(f_r(x_r)-f_r(x_\tau^\opt)\bigr)
\nonumber\\
&>
-\frac N2
+
32N\frac3{128}
=
\frac N4.
\label{eq:phase-regret-lb}
\end{align}

By the assumed prefix-regret guarantee,
\[
R_{\mathcal A}(\tau)
\le
R(\tau)
=
R(33N).
\]
Since
\[
R(t)=o(t),
\]
we have
\[
\frac{R(33N)}{N}
=
33\,\frac{R(33N)}{33N}
\longrightarrow 0
\qquad\text{as }N\to\infty.
\]
Hence, there exists $N_0$ such that, for every $N\ge N_0$,
\[
R(33N)<\frac N4.
\]
But from \eqref{eq:phase-regret-lb}, we have that if the total movement during the phase
is less than $1/4$, then
\[
R_{\mathcal A}(33N)>\frac N4.
\]
Therefore, for every phase with
\[
N=N_{k-1}\ge N_0,
\]
we must have
\[
\sum_{t=N+1}^{N+32N}
\|x_t-x_{t-1}\|
\ge \frac14.
\]

Since
\[
N_{k-1}=33^{k-1},
\]
only finitely many phases satisfy
\[
N_{k-1}<N_0.
\]
In particular, the number of such phases is at most
\[
1+\left\lceil \log_{33}N_0\right\rceil,
\]
which is independent of $T$. Hence all but finitely many complete phases contribute at least
$1/4$ movement, where the number of exceptional phases is independent
of $T$.

Finally, the number of complete phases contained in $[T]$ is
\[
K=\left\lfloor\log_{33}T\right\rfloor.
\]
Therefore, denoting $O(1)$ as a constant for the finitely many phase where we do not have a lower bound on the movement cot, 
\[
M_{\mathcal A}(T)
\ge
\frac14K-O(1)
=
\Omega(\log T).
\]

%
\end{proof}

\section{Proof of Theorem \ref{thm:ubcvx}}
\begin{proof}
Let $g_t\in\partial f_t(x_t)$. Since $x^\opt\in S_T\subseteq S_t$ and
$f_t$ is convex,
\begin{equation}\label{eq:convexcond}
f_t(x_t)-f_t(x^\opt)
\le
\langle g_t,x_t-x^\opt\rangle.
\end{equation}

The first-order optimality condition for algorithm \textrm{Prox} \eqref{defn:algprox} implies
\[
\left\langle
g_t+\frac{x_t-x_{t-1}}{\eta},
x^\opt-x_t
\right\rangle\ge0,
\]
since $x^\opt\in S_t$ for all $t$. 
Similar to \eqref{eq:proxmanipulation}, we have 
\[
\langle g_t,x_t-x^\opt\rangle
\le
\frac{
\|x_{t-1}-x^\opt\|^2
-
\|x_t-x^\opt\|^2
-
\|x_t-x_{t-1}\|^2
}{2\eta}.
\]
Therefore, from \eqref{eq:convexcond}, we get
\[
f_t(x_t)-f_t(x^\opt)
\le
\frac{
\|x_{t-1}-x^\opt\|^2
-
\|x_t-x^\opt\|^2
-
\|x_t-x_{t-1}\|^2
}{2\eta}.
\]
Dropping the nonpositive last term and summing over $t$ gives
\[
\mathrm{Regret}_\cA(T)
\le
\frac{
\|x_0-x^\opt\|^2
-
\|x_T-x^\opt\|^2
}{2\eta}
\le
\frac{D^2}{2\eta}
=
\frac{\epsilon T}{2}.
\]

Using Lemma \ref{lem:algequiv}, we know that the algorithm \textrm{Prox}'s update \eqref{defn:algprox} is of the form
\[
x_t=\Pi_{S_t}(x_{t-1}-e_t),
\qquad
e_t:=\eta g_t.
\]
Since $f_t$ is $G$-Lipschitz,
\[
\|e_t\|\le \eta G.
\]
Since $S_t$'s are nested, applying Lemma \ref{lem:totalpertprojlength}, we get
\[
M_\cA(T)
\le
C_d\left(
D+\sum_{t=1}^T\|e_t\|
\right),
\]
and therefore
\[
M_\cA(T)
\le
C_d(D+\eta GT)
=
C_d\left(
D+\frac{GD^2}{\epsilon}
\right).
\]
\end{proof}
\section{Proof of Theorem \ref{lemma:lbgenConvex}}\label{app:lemma:lbgenConvex}
\label{app:lemma:lbgenConvex}

\begin{proof}
Let
\[
\mathcal X=[-1,1]\times[0,1], \]
and the fixed loss function be 

\[
f(x,y)=-y.
\]
Thus, $d=2$, $f$ is linear and hence convex, and
\[
\|\nabla f(x,y)\|=1,
\]
so $f$ is $1$-Lipschitz on $\mathcal X$.

Let
\[
\operatorname{conv}(A)
\]
denote the convex hull of a set $A$, i.e., the smallest convex set
containing $A$.

Let
\[
R:=\mathrm{Regret}_{\cA}(T).
\]
Define
\[
\delta := 2\sqrt{\frac{R+1}{T}},
\qquad
K := \left\lfloor \frac{1}{2\delta}\right\rfloor .
\]

The assumption
\[
T\ge16\bigl(R+1\bigr)
\]
gives
\[
\delta\le\frac12.
\]
Hence
\[
\frac{1}{2\delta}\ge1,
\]
and therefore
\[
K
=
\left\lfloor\frac{1}{2\delta}\right\rfloor
\ge
\frac{1}{4\delta}.
\]

For $k=1,\ldots,K$, define
\[
q_k:=1-(k-1)\delta,
\qquad
s_k:=(-1)^k,
\qquad
z_k:=(s_k,q_k),
\]
and let
\[
c:=(0,1-K\delta).
\]
Since $K\delta\le1/2$, we have
\[
q_k\ge q_K=1-(K-1)\delta>\frac12,
\]
and
\[
\frac12\le1-K\delta\le1.
\]
Thus,
\[
z_1,\ldots,z_K,c\in\mathcal X.
\]

For $k=1,\ldots,K$, define
\[
S_k
:=
\operatorname{conv}\{c,z_k,z_{k+1},\ldots,z_K\}.
\]
Then
\[
S_1\supseteq S_2\supseteq\cdots\supseteq S_K.
\]
Moreover,
\[
q_k>q_{k+1}>\cdots>q_K>1-K\delta.
\]
Hence, the unique minimizer of $f(x,y)=-y$ over $S_k$ is
$z_k$. Thus
\[
v_k:=\min_{x\in S_k}f(x)=f(z_k)=-q_k.
\]

For $k=1,\ldots,K$, define the {\it trigger} set
\[
A_k
:=
\left\{
(x,y)\in S_k:
s_kx\ge\frac12
\right\}.
\]

See Fig. \ref{fig:lb} for an illustration of this input.

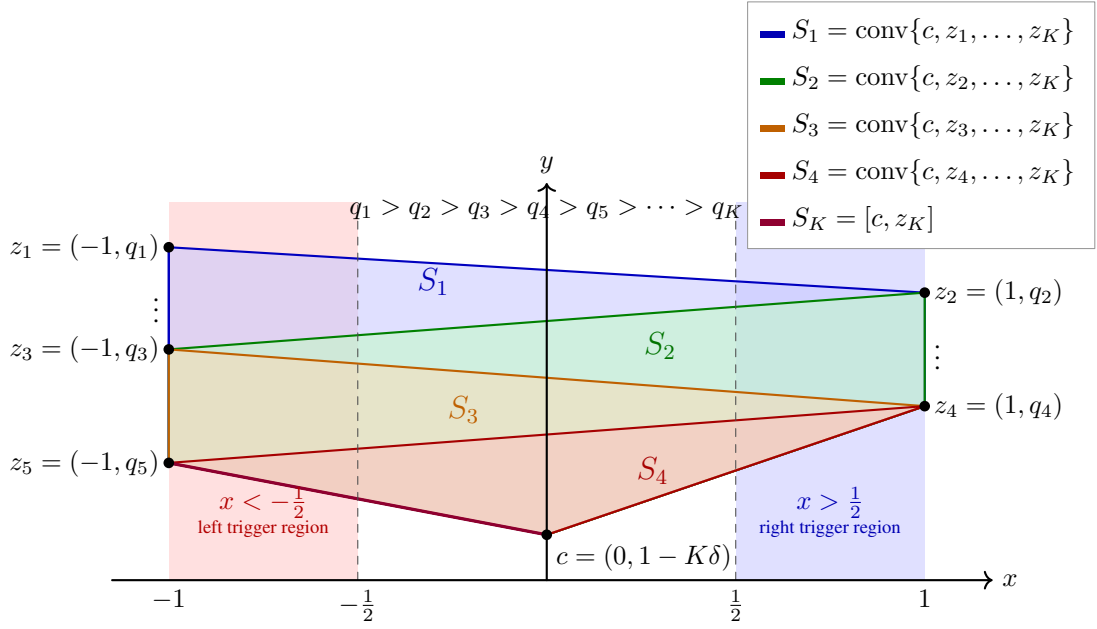
\begin{figure}[t]
\centering
\begin{tikzpicture}[
    x=5.0cm,
    y=5.0cm,
    line join=round,
    line cap=round
]

\coordinate (z1) at (-1,0.88);
\coordinate (z2) at ( 1,0.76);
\coordinate (z3) at (-1,0.61);
\coordinate (z4) at ( 1,0.46);
\coordinate (z5) at (-1,0.31);
\coordinate (c)  at ( 0,0.12);

\fill[red!12] (-1,0) rectangle (-0.5,1);
\fill[blue!12] (0.5,0) rectangle (1,1);

\fill[blue!22,opacity=.55]
    (z1) -- (z2) -- (z4) -- (c) -- (z5) -- cycle;

\draw[blue!75!black,thick]
    (z1) -- (z2) -- (z4) -- (c) -- (z5) -- cycle;

\draw[blue!75!black,thick]
    (z1) -- (z5);

\fill[green!25,opacity=.50]
    (z3) -- (z2) -- (z4) -- (c) -- (z5) -- cycle;

\draw[green!50!black,thick]
    (z3) -- (z2) -- (z4) -- (c) -- (z5) -- cycle;

\fill[orange!30,opacity=.50]
    (z3) -- (z4) -- (c) -- (z5) -- cycle;

\draw[orange!75!black,thick]
    (z3) -- (z4) -- (c) -- (z5) -- cycle;

\fill[red!25,opacity=.50]
    (z5) -- (z4) -- (c) -- cycle;

\draw[red!65!black,thick]
    (z5) -- (z4) -- (c) -- cycle;

\draw[purple!75!black,very thick]
    (z5) -- (c);

\draw[dashed,black!70]
    (-0.5,0) -- (-0.5,1);

\draw[dashed,black!70]
    (0.5,0) -- (0.5,1);

\draw[->,thick]
    (-1.15,0) -- (1.18,0)
    node[right] {$x$};

\draw[->,thick]
    (0,0) -- (0,1.05)
    node[above] {$y$};

\node[below] at (-1,0) {$-1$};
\node[below] at (-0.5,0) {$-\frac12$};
\node[below] at (0.5,0) {$\frac12$};
\node[below] at (1,0) {$1$};

\node[red!70!black,align=center]
    at (-0.75,0.18)
    {$x<-\frac12$\\[-1mm]
     \scriptsize left trigger region};

\node[blue!70!black,align=center]
    at (0.75,0.18)
    {$x>\frac12$\\[-1mm]
     \scriptsize right trigger region};

\fill (z1) circle (2pt);
\fill (z2) circle (2pt);
\fill (z3) circle (2pt);
\fill (z4) circle (2pt);
\fill (z5) circle (2pt);
\fill (c)  circle (2pt);

\node[left] at (z1) {$z_1=(-1,q_1)$};
\node[right] at (z2) {$z_2=(1,q_2)$};

\node[left] at (z3) {$z_3=(-1,q_3)$};
\node[right] at (z4) {$z_4=(1,q_4)$};

\node[left] at (z5) {$z_5=(-1,q_5)$};

\node[below right] at (c)
    {$c=(0,1-K\delta)$};

\node[left] at (-1,0.735) {$\vdots$};
\node[right] at (1,0.61) {$\vdots$};

\node at (0,0.98)
    {$q_1>q_2>q_3>q_4>q_5>\cdots>q_K$};

\node[blue!70!black,font=\large] at (-0.30,0.79)
    {$S_1$};

\node[green!45!black,font=\large] at (0.30,0.62)
    {$S_2$};

\node[orange!75!black,font=\large] at (-0.22,0.45)
    {$S_3$};

\node[red!65!black,font=\large] at (0.28,0.29)
    {$S_4$};

\node[
    draw=black!40,
    fill=white,
    align=left,
    anchor=north west,
    inner sep=5pt
] at (0.53,1.53)
{
\begin{tabular}{@{}l@{}}
$\color{blue!70!black}\rule{0.32cm}{0.08cm}$
$S_1=\operatorname{conv}\{c,z_1,\ldots,z_K\}$\\[2mm]
$\color{green!50!black}\rule{0.32cm}{0.08cm}$
$S_2=\operatorname{conv}\{c,z_2,\ldots,z_K\}$\\[2mm]
$\color{orange!75!black}\rule{0.32cm}{0.08cm}$
$S_3=\operatorname{conv}\{c,z_3,\ldots,z_K\}$\\[2mm]
$\color{red!65!black}\rule{0.32cm}{0.08cm}$
$S_4=\operatorname{conv}\{c,z_4,\ldots,z_K\}$\\[2mm]
$\color{purple!75!black}\rule{0.32cm}{0.08cm}$
$S_K=[c,z_K]$
\end{tabular}
};

\end{tikzpicture}
\caption{
Illustration of the nested feasible sets
$S_k=\operatorname{conv}\{c,z_k,\ldots,z_K\}$.
The points alternate between $x=-1$ and $x=1$ with
$q_1>q_2>\cdots>q_K$. In the illustrated $K=5$ case,
$z_3$ and $z_5$ lie on the left boundary of $S_1$, while
$z_2$ and $z_4$ lie on its right boundary. The dashed lines
$x=\pm\frac12$ indicate the regions used by the trigger sets
$A_k=\{(x,y)\in S_k:s_kx\ge\frac12\}$.
}
\label{fig:lb}
\end{figure}

\begin{lemma}\label{lem:lbp}
For every
\[
p=(x,y)\in S_k\setminus A_k,
\]
we have
\[
f(p)-f(z_k)>\frac{\delta}{4}.
\]
\end{lemma}
Proof of Lemma \ref{lem:lbp} is provided at the end.
The purpose of Lemma \ref{lem:lbp} is to show that every action
 outside the trigger set $A_k$ incurs a definite loss relative to the optimal action $z_k$ in $S_k$. Thus, an algorithm cannot remain outside
$A_k$ for too many consecutive rounds: weak adaptivity forces it to
enter $A_k$ after a bounded number of rounds. We will use this to define the input in phases as follows.

Phase $k$ uses the feasible set $S_k$. The first phase starts at
round $1$. More generally, if phase $k$ ends at round $\tau_k$, then
phase $k+1$ begins at round $\tau_k+1$ with feasible set $S_{k+1}$.
Thus, if round $t$ belongs to phase $k$, we set
\[
S_t:=S_k.
\]
Phase $k$ continues until the first round $t$ for which
\[
x_t\in A_k.
\]
That round is the final round of phase $k$.


Let
\begin{equation}\label{defn:L}
L:=
\left\lceil
\frac{4(R+1)}{\delta}
\right\rceil.
\end{equation}

Suppose, for contradiction, that phase $k$ contains $L$ consecutive
rounds without the algorithm $\cA$ choosing an action in $A_k$. Let
$I_k$ denote these $L$ rounds, and {\bf let $b$ be the last round of $I_k$}.
Since the feasible set is $S_k$ throughout this phase,
\[
x_b^\opt
\in
\arg\min_{x\in S_b}\sum_{t=1}^b f(x)
=
\arg\min_{x\in S_k} b\,f(x)
=
\{z_k\},
\]
and hence
\[
x_b^\opt=z_k.
\]

Therefore, the weakly adaptive regret guarantee applied to $I_k$ gives
\[
\sum_{t\in I_k}
\bigl(f(x_t)-f(z_k)\bigr)
\le R.
\]
On the other hand, $x_t\notin A_k$ for every $t\in I_k$, so by
Lemma~\ref{lem:lbp},
\[
\sum_{t\in I_k}
\bigl(f(x_t)-f(z_k)\bigr)
>
\frac{\delta |I_k|}{4}
=
\frac{\delta L}{4}
\ge R+1,
\]
where the last inequality follows from the definition of $L$. This is a contradiction. Hence every phase has length at most $L$.

It remains to verify that all $K$ phases fit within the horizon. 
Since each phase contains at most $L$ rounds,
\[
KL
\le
\frac{1}{2\delta}
\left(
\frac{4(R+1)}{\delta}+1
\right).
\]
Using
\[
\delta^2=\frac{4(R+1)}{T},
\]
we have
\[
\frac{2(R+1)}{\delta^2}
=
\frac{T}{2}.
\]
Moreover,
\[
\frac{1}{2\delta}
\le
\frac{\sqrt{T}}{4}
\le
\frac{T}{4},
\]
where the first inequality follows from
\[
\delta\ge\frac{2}{\sqrt{T}}
\]
and the second from $T\ge1$. Therefore,
\[
KL
\le
\frac{T}{2}+\frac{T}{4}
=
\frac{3T}{4}
<T.
\]
Hence all $K$ phases fit within the horizon; the remaining rounds, if
any, are filled by continuing with the feasible set $S_K$.

Finally, since
\[
s_{k+1}=-s_k,
\]
we have
\[
A_k
\subseteq
\left\{(x,y):s_kx\ge\frac12\right\},
\]
while
\[
A_{k+1}
\subseteq
\left\{(x,y):s_{k+1}x\ge\frac12\right\}
=
\left\{(x,y):s_kx\le-\frac12\right\}.
\]
Thus, 
\[
\operatorname{dist}(A_k,A_{k+1}) = \min_{p_1\in A_k,p_2\in A_{k+1}}||p_1-p_2||\ge1.
\]
Every phase $k$ ends with an action $x_{\tau_k}\in A_k$.
Since the next phase ends with an action
$x_{\tau_{k+1}}\in A_{k+1}$ and
\[
\operatorname{dist}(A_k,A_{k+1})\ge 1,
\]
we have
\[
1
\le
\|x_{\tau_k}-x_{\tau_{k+1}}\|
\le
\sum_{t=\tau_k+1}^{\tau_{k+1}}
\|x_t-x_{t-1}\|.
\]
Thus each pair of consecutive phases contributes at least one unit of
movement, and hence
\[
M_{\cA}(T)\ge K-1.
\]
Using
\[
K\ge\frac1{4\delta},
\]
we obtain
\[
M_\cA(T)
\ge
\frac1{4\delta}-1
=
\frac18
\sqrt{\frac{T}
{R+1}}
-1.
\]
In particular, if
\[
R=o(T),
\]
then
\[
\sqrt{\frac{T}{R+1}}
\longrightarrow\infty.
\]
Hence, for all sufficiently large $T$,
\[
\frac18\sqrt{\frac{T}{R+1}}-1
\ge
\frac1{16}\sqrt{\frac{T}{R+1}},
\]
and therefore
\[
M_\cA(T)
=
\Omega\!\left(
\sqrt{\frac{T}{R+1}}
\right).
\]
\end{proof}

\begin{proof}[Proof of Lemma \ref{lem:lbp}]
Since $p\in S_k$, there exist coefficients
\[
\lambda_0,\lambda_k,\ldots,\lambda_K\ge0,
\qquad
\lambda_0+\sum_{j=k}^K\lambda_j=1,
\]
such that
\[
p
=
\lambda_0c+\sum_{j=k}^K\lambda_jz_j.
\]
Thus, $p$ is a convex combination of $c$ and the points
$z_k,\ldots,z_K$.

Let
\[
w
:=
\lambda_0+
\sum_{\substack{j>k\\s_j=-s_k}}\lambda_j.
\]
This is the total weight assigned to $c$ and to those points
$z_j$ with $j>k$ whose first coordinate has the opposite sign from
$s_k$. These are precisely the terms that reduce the quantity $s_kx$.

Since
\[
p=\lambda_0c+\sum_{j=k}^K\lambda_jz_j
\]
and
\[
c=(0,1-K\delta),
\qquad
z_j=(s_j,q_j),
\]
the first coordinate of $p=(x,y)$ is
\[
x
=
\lambda_0\cdot 0+\sum_{j=k}^K\lambda_j s_j
=
\sum_{j=k}^K\lambda_j s_j.
\]
Multiplying both sides by $s_k\in\{-1,1\}$ gives
\[
s_kx
=
\sum_{j=k}^K\lambda_j s_ks_j.
\]
For $j=k$, we have
\[
s_ks_k=1,
\]
so the corresponding term is $\lambda_k$. For $j>k$, either
$s_j=s_k$, in which case
\[
s_ks_j=1,
\]
or $s_j=-s_k$, in which case
\[
s_ks_j=-1.
\]
Therefore,
\[
\begin{aligned}
s_kx
&=
\lambda_k
+
\sum_{\substack{j>k\\s_j=s_k}}\lambda_j
-
\sum_{\substack{j>k\\s_j=-s_k}}\lambda_j, \\
&=
1-\lambda_0
-
2\sum_{\substack{j>k\\s_j=-s_k}}\lambda_j.
\end{aligned}
\]
Since $p\notin A_k$,
\[
s_kx<\frac12.
\]
Therefore,
\[
\lambda_0
+
2\sum_{\substack{j>k\\s_j=-s_k}}\lambda_j
>
\frac12.
\]
Since
\[
2w
=
2\lambda_0
+
2\sum_{\substack{j>k\\s_j=-s_k}}\lambda_j
\ge
\lambda_0
+
2\sum_{\substack{j>k\\s_j=-s_k}}\lambda_j,
\]
we obtain
\[
w>\frac14.
\]

Every point contributing to $w$ has second coordinate at most
$q_k-\delta$. Indeed, for every $j>k$,
\[
q_j\le q_k-\delta,
\]
while
\[
1-K\delta
\le
q_k-\delta
\]
because $k\le K$. All remaining points in the convex combination
have second coordinate at most $q_k$. Hence
\[
\begin{aligned}
y
&\le
w(q_k-\delta)+(1-w)q_k\\
&=
q_k-w\delta\\
&<
q_k-\frac{\delta}{4}.
\end{aligned}
\]
Since
\[
f(p)-f(z_k)
=
-y+q_k,
\]
we conclude that
\[
f(p)-f(z_k)>\frac{\delta}{4}.
\]
\end{proof}

\end{document}